%% file: main.tex
\documentclass[12pt, final]{l4dc2022}

\title[Online No-regret Model-Based Meta RL for Personalized Navigation]{Online No-regret Model-Based Meta RL for Personalized Navigation}
\usepackage{times}

\usepackage{algorithm}
\usepackage{algorithmic}
\usepackage{babel}
\usepackage{setspace}
\usepackage[font=small,labelfont=bf]{caption}
\usepackage{makecell}

\usepackage{mkolar_definitions,comment}
\usepackage{booktabs}

\usepackage{wrapfig}

\author{%
 \Name{Yuda Song} \textsuperscript{\rm 1} \Email{yudas@andrew.cmu.edu}\\
 \Name{Ye Yuan}  \textsuperscript{\rm 1}  \Email{yyuan2@cs.cmu.edu}\\
 \Name{Wen Sun}  \textsuperscript{\rm 2}  \Email{ws455@cornell.edu}\\
 \Name{Kris Kitani} \textsuperscript{\rm 1} \Email{kkitani@cs.cmu.edu}\\
 \addr \textsuperscript{\rm 1} Carnegie Mellon University, \textsuperscript{\rm 2} Cornell University
}

\begin{document}

\maketitle

\begin{abstract}%
The interaction between a vehicle navigation system and the driver of the vehicle can be formulated as a model-based reinforcement learning problem, where the navigation systems (agent) must quickly adapt to the characteristics of the driver (environmental dynamics) to provide the best sequence of turn-by-turn driving instructions. Most modern day navigation systems (e.g, Google maps, Waze, Garmin) are not designed to personalize their low-level interactions for individual users across a wide range of driving styles (e.g., vehicle type, reaction time, level of expertise). Towards the development of personalized navigation systems that adapt to a variety of driving styles, we propose an online no-regret model-based RL method that quickly conforms to the dynamics of the current user. As the user interacts with it, the navigation system quickly builds a user-specific model, from which navigation commands are optimized using model predictive control. By personalizing the policy in this way, our method is able to give well-timed driving instructions that match the user's dynamics. Our theoretical analysis shows that our method is a no-regret algorithm and we provide the convergence rate in the agnostic setting. Our empirical analysis with 60+ hours of real-world user data using a driving simulator shows that our method can reduce the number of collisions by more than 60\%.\end{abstract}

\vspace{0.2cm}
\begin{keywords}%
  Model-based RL, Online Learning, System Identification, Personalization
\end{keywords}


\input{intro}
\input{related}
\input{pre}

\input{alg}

\input{analysis_short}
\input{exp}

\section{Summary}
In this work we propose an online meta MBRL algorithm that can adapt to different incoming dynamics in an online fashion. We proved that our algorithm is no-regret. We specifically study the application of the proposed algorithm in personalized voice navigation and develop a machine learning system, which could also be beneficial for future research. The extensive real-world use study shows that our algorithm could indeed efficiently adapt to unseen dynamics, and also improve the safety of the voice navigation. 



\clearpage
\bibliography{main}

\newpage
\onecolumn
\appendix
\input{appendix}

\end{document}

%% file: intro.tex
\section{Introduction}
\label{sec:intro}

Reinforcement Learning (RL) and Markov Decision Processes (MDP) \cite{sutton2018reinforcement} provide a natural paradigm for training an agent to build a navigation system. That is, by treating the navigation system that delivers instructions as the \textit{agent} and combining the vehicle and the user as the \textit{environment}, we can design RL algorithms to train a navigation agent that provides proper instructions by maximizing a reward of interest. For example, we could reward the agent when it delivers accurate and timely instructions. There has been a rich line of research on applying RL in various navigation tasks, including robot navigation \cite{surmann2020deep, liu2020robot}, indoor navigation \cite{chen2020soundspaces,chen2021semantic}, navigation for blind people~\cite{ohnbar2018personalized}, etc. In this work, we put our focus on vehicle navigation, a task also explored by the previous literature \cite{ kiran2021deep, stafylopatis1998autonomous, deshpande2019deep}. 

\begin{figure}
    \centering
    \begin{minipage}{0.5\textwidth}
        \centering
    \includegraphics[width=1\linewidth]{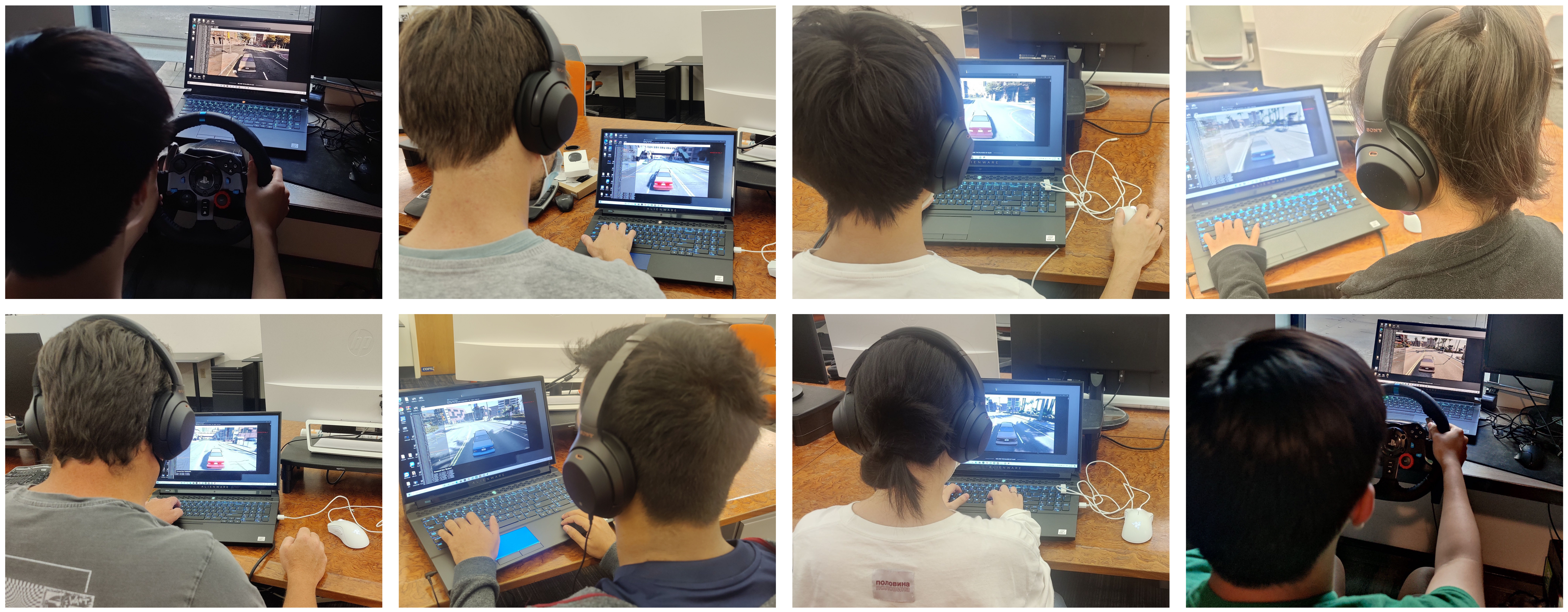}
    \caption{Different participants using different driving kits in our navigation system built on Carla. Each participant at least interacts with the system for three hours.}
    \label{fig:people}
    \end{minipage}\hfill
    \begin{minipage}{0.45\textwidth}
    \centering
    \includegraphics[width=0.48\textwidth]{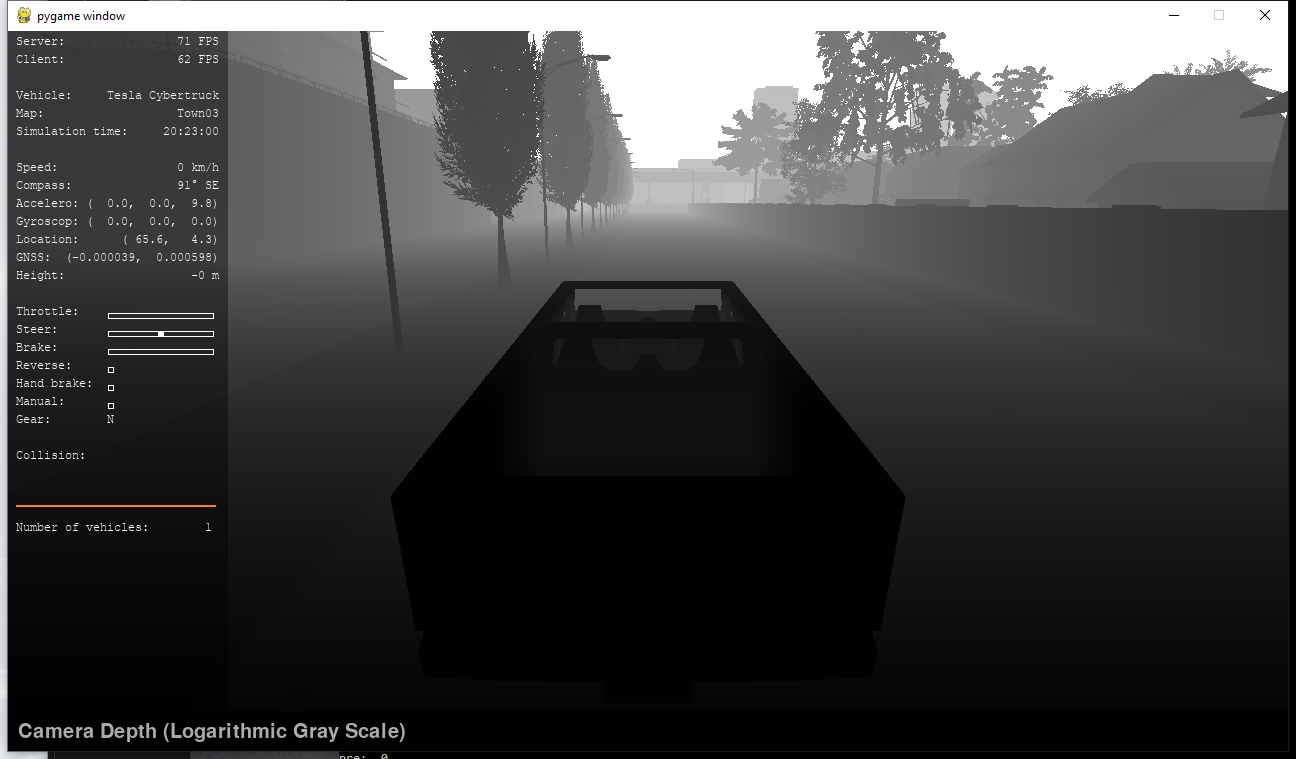}
    \includegraphics[width=0.48\textwidth]{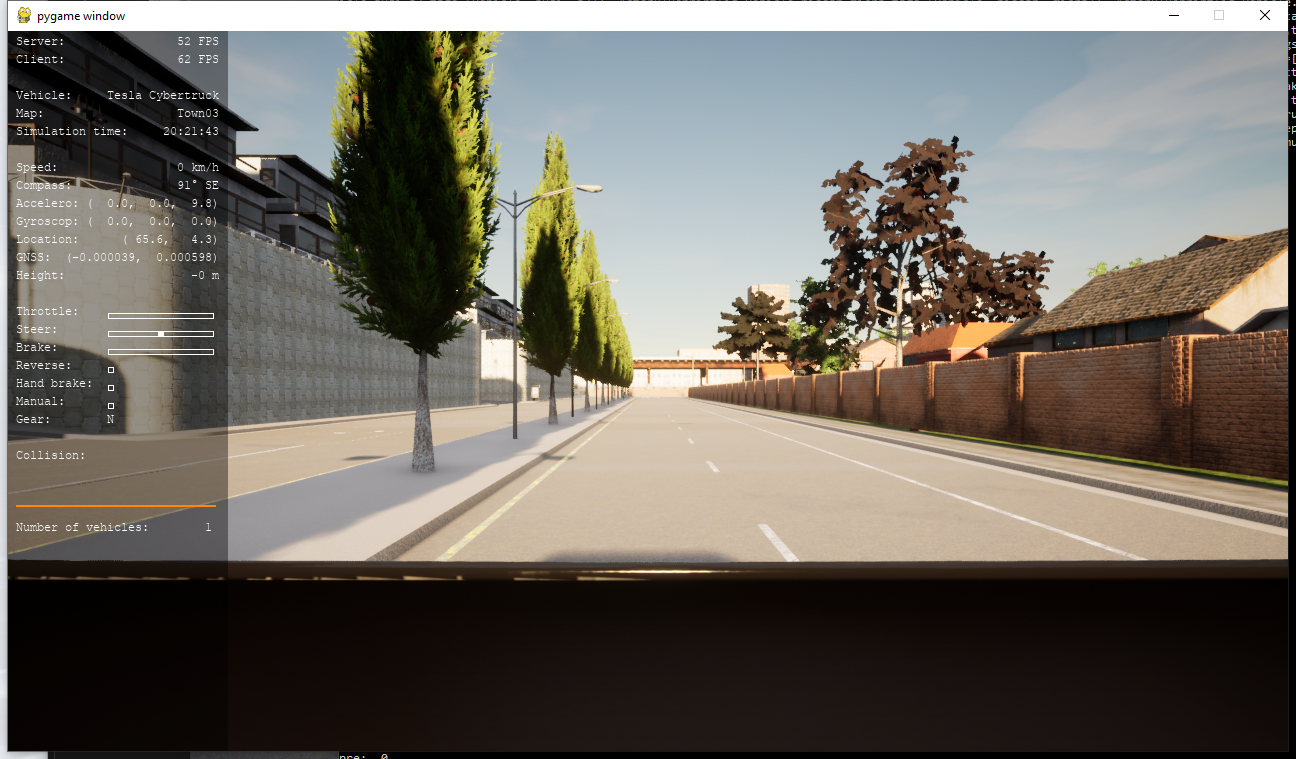}
    \caption{Examples of dynamics changes: left: using different vehicles (cybertruck) and different visual situations (mimicking foggy situations). Right: using a first-person view instead of a third-person view, while the latter is easier to control.}
    \label{fig:dynamics_changes}
    \end{minipage}
\end{figure}

However, most of the previous works only consider a static scenario, where the user (robot, vehicle or human) that receives the instructions does not change its dynamics during the deployment phases. This is not practical in the real world - in practice, the navigation system is faced with different users with changing or unseen dynamics during deployment (in our case, we could face different drivers, vehicles and weather conditions everyday), which requires fast personalization of the navigation system. However, there is little hope learning a monotonic policy that fits all (and unseen) users. In fact, even well-established navigation systems such as Google Maps are only designed for the average user. Thus, personalization needs to be achieved through adaptation, and we need a policy that can quickly adapt to the incoming user. Recently, meta learning \cite{finn2017model,duan2016rl,finn2017one} presents a way to train policies that adapt to new tasks by few-shot adaptation. Similarly, in our online setting, we need to focus on sample efficiency in order to have performance guarantee, i.e., we want to adapt with as few samples as possible - otherwise, a policy is sub-optimal as long as the adaptation is unfinished. Towards this direction, model-based RL methods have shown promising sample complexity/efficiency results (compared to model-free algorithms) in both theoretical analysis \cite{sun2019model,tu2019gap} and empirical applications \cite{chua2018deep,DBLP:journals/corr/abs-1907-02057}. 

 Drawing the above intuitions from meta RL and MBRL, in this work we propose an online model-based meta RL algorithm for building a personalized vehicle navigation system. Our algorithm trains a meta dynamics model that can quickly adapt to the ground truth dynamics of each incoming user, and a policy can easily be induced by planning with the adapted model. Theoretically, we show that our algorithm is no-regret. Empirically, we develop a navigation system inside the Carla simulator \cite{dosovitskiy2017carla} to evaluate our algorithm. The state space of the navigation system is the specification of the vehicle (e.g., location, velocity) and the action space is a set of audio instructions. A policy will be trained and tested inside the system with a variety of scenarios (driver, vehicle type, visual condition, level of skills). Ideally, a personalized policy should deliver timely audio instructions to each different user. We conduct extensive experiments (60+ hours) on real human participants with diverse driving experiences as shown in Fig.~\ref{fig:people}, which is significantly more challenging than most previous works that only perform experiments with simulated agents instead of real humans \cite{koh2020real, nagabandi2018deep}. The user study shows that our algorithm outperforms the baselines in following optimal routes and it is better at enforcing safety where the collision rate is significantly reduced by more than 60\%. We further highlight that the proposed algorithm is able to generate new actions to facilitate the navigation process. 
 
Our contributions are as follows: 1) we propose an algorithm that guarantees an adaptive policy that achieves near-optimal performance in a stream of changing dynamics and environments. 2) Theoretically, we prove our algorithm is no-regret. 3) We build a vehicle navigation system inside the state-of-the-art driving simulator that could also facilitate future research in this direction. 4) We perform extensive real-human experiments which demonstrate our method's strong ability to adapt to user dynamics and improve safety.

%% file: related.tex
\section{Related Works}
\subsection{RL for Navigation System}
The reinforcement learning paradigm fits naturally into the navigation tasks and has encouraged a wide range of practical studies. In this work we focus on vehicle navigation \cite{koh2020real, kiran2021deep, stafylopatis1998autonomous, deshpande2019deep}. Previous work can be separated into two major categories: navigation with human interaction \cite{thomaz2006reinforcement, wang2003user, hemminahaus2017towards} and navigation without human interaction~\cite{kahn2018self, ross2008bayesian}. The first category includes tasks such as indoor navigation and blind person navigation. Such tasks require considering personalization issues when facing different users, who potentially bring drastically different dynamics to the system. Previously, vehicle navigation is categorized into the second field since there is no human interaction. However, in our work, we provide a novel perspective and introduce new challenges by incorporating human drivers into the system. Thus in order for an algorithm to solve this new task, it should consider personalization, accuracy and safety at the same time. In this work, we theoretically and empirically show that our proposed solution could address the three key components simultaneously.

\subsection{Model-based RL and Meta RL}
Model-based RL and meta RL are designed to achieve efficient adaptation. MBRL aims to model the dynamics of the environments and plan in the learned dynamics models. Prior works \cite{sun2019model,tu2019gap,chua2018deep,song2021pc} have shown superior sample efficiency of model-based approaches than their model-free counterparts, both in theory and in practice. For navigation with human interaction, prior works have shown promising results by modeling human dynamics \cite{torrey2013robot, iqbal2016tempo, daniele2017navigational}. However, these works do not provide any personalized navigation.

Towards the adaptability side, meta RL approaches~\cite{finn2017model,finn2017one,duan2016rl} have demonstrated the ability to adapt to a new task during test time by few-shot adaptation. Previously, there are also works that propose meta-learning algorithms for model-based RL \cite{nagabandi2018deep, nagabandi2018learning, lin2020model, clavera2018model, belkhale2021model, saemundsson2018meta}. However, these methods do not have any theoretical guarantee. Additionally, it is unclear if these works can adapt in real-world scenarios since their evaluation is based only on simulation or non-human experiments. In contrast, the proposed method is evaluated with extensive user study (60+ hours) with real-human participants which demonstrates the method's ability to deliver personalized driving instructions that quickly adapts to the ever-changing dynamics of the user.

%% file: pre.tex

\section{Preliminaries}
\label{sec:pre}

A finite horizon (undiscounted) Markov Decision Processes (MDP) is defined by $\mathcal{M} = \{\Scal, \Acal, P, H, \\ C, \mu\}$, where $\Scal$ is the state space, $\Acal$ is the action space, $P: \Scal \times \Acal \to \Delta(\Scal)$ is the transitional dynamics kernel, $C: \Scal \times \Acal \to \mathbb{R}$ is the cost function, $H$ is the horizon and $\mu \in  \Delta(\Scal)$ is the initial state distribution. In this work, we consider an online learning setup where we encounter a stream of MDPs $\{\Mcal^{(t)}\}_{t=1}^T$. Each $\Mcal^{(t)}$ is defined by $\{\Scal, \Acal, P^{(t)}, H,C, \mu\}$, i.e., the transitional dynamics kernels are different across the MDPs. The goal is to learn a stochastic policy $\pi: \Scal \to \Delta(\Acal)$ that minimizes the total sum of costs. Denote $d_{h,\mu,P^{(t)}}^{\pi}(s,a) = \EE_{s \sim \mu}\left[\PP(s_h = s, a_h = a | s_0 = x, \pi, P^{(t)})\right]$ to be the probability of visiting $(s,a)$ under policy $\pi$, initial distribution $\mu$ and transition $P^{(t)}$.  Let $d_{\mu,P^{(t)}}^{\pi} = \frac{1}{H}\sum_{h=0}^{H-1} d_{\mu,h,P^{(t)}}^{\pi}$ be the average state-action distribution. Since in the online setting, if the algorithm incurs a sequence of (adapted) policies $\boldsymbol{\pi} = \{\pi^{(t)}\}_{t=1}^T$, the value at round $t$ is $ V^{(t)}(s) = \EE \left[\sum_{h=0}^{H-1} C(s_h,a_h | s_0 = s,\pi^{(t)},  P^{(t)})\right]$. Then $J(\boldsymbol{\pi}) = \frac{1}{T} \sum_{t=1}^T \EE_{s \sim \mu}\left[ V^{(t)}(s) \right]$ is the objective function.  We also consider the best policy from hindsight to be adaptive: let $\boldsymbol{\pi^\ast} = \{\pi^{\ast(t)}\}_{t=1}^T$ be a sequence of policies adapted from a meta policy (or model in MBRL setting) $\pi^\ast$ (which we will define rigorously in Sec.~\ref{sec:analysis}), and we mildly abuse notation here by letting $\Pi$ be the class of regular policies and meta policies, then we define the regret of an algorithm up to time $T$ as:
$R(T) = J(\boldsymbol{\pi}) - \min_{\pi^\ast \in \Pi}J(\boldsymbol{\pi^\ast})$,
and the goal is to devise an algorithm whose regret diminishes with respect to T.

Finally we introduce the coverage coefficient term for a policy $\pi$: $c_{\nu}^{\pi} = \sup_{s,a} \frac{d_{\mu}^{\pi}(s,a)}{\nu(s,a)},$ where $\nu$ is some state-action distribution. This is the maximum mismatch between an exploration distribution $\nu$ and the state-action distribution induced by a policy $\pi$.

%% file: alg.tex
\section{Online Meta Model-based RL}

In this section we present our algorithm: Online Meta Model-based Reinforcement Learning, a combination of online model-based system identification paradigm and meta model-based reinforce-
ment learning. We provide the pseudocode of our algorithm in Alg.~\ref{alg:alg_box}, Appendix Sec.~\ref{app:alg}. 

Our overall goal is to learn a meta dynamics model $\widehat{P}: \Scal \times \Acal \to \Delta(\Scal)$ in an online manner, such that if we have access to a few samples $D = \{s_d,a_d,s'_d\}_{d=1}^{|D|} \sim d_{P^{(t)}}^{\pi}$ induced by some policy $\pi$ in the environment of interest $\Mcal^{(t)}$, we can quickly adapt to the new environment by performing one-shot gradient descent on the model-learning loss function $\ell$ (such as MLE): $U(\widehat{P},D) := \widehat{P} - \alpha_{\text{adapt}} \frac{1}{|D|} \sum_{d=1}^{|D|} \nabla\ell(\widehat{P}(s_d,a_d),s'_d)$, which will be an accurate estimation on the ground truth dynamics $P^{(t)}$. Here the gradient is taken with respect to the parameters of $\widehat{P}$ and $\alpha_{\text{adapt}}$ is the adaptation learning rate. This remains for the rest of this section. With such models, we then can plan (running policy optimization methods or using optimal control methods) to obtain a policy for each incoming environment. 

Specifically, we assume that we begin with an exploration policy $\pi^e$ which could be a pre-defined non-learning-based policy. Such an exploration policy is important in two ways: first is to collect samples for the adaptation process and second is to cover the visiting distribution for a good policy for the theoretical analysis.  We also assume we have access to an offline data set $D_{\text{off}}$ which includes transition samples from a set of MDPs $\{\Mcal_j\}_{j=1}^M$, where $M$ should be much smaller than $T$. In practice, we can simply build the offline dataset by rolling out the explore policy in $\{\Mcal_j\}_{j=1}^M$.
Our algorithm starts with first training a warm-started model $\widehat{P}^{(1)}$ with the offline dataset:
\begin{align}
    \widehat{P}^{(1)} = \argmax_{P \in \Pcal} \sum_{s,a,s' \in D_{\text{off}}} \log P(s'|s,a),
\label{eq:offline}
\end{align}
where $\Pcal$ is the model class.

Next during the online phase when every iteration $t$ a new $\Mcal^{(t)}$ comes in, at the beginning of this iteration we first collect one trajectory $\tau_t$ with the exploration policy $\pi^e$, then we perform one shot adaptation on our latest meta model $\widehat{P}^{(t)}$ to obtain $U(\widehat{P}^{(t)}), \tau_t)$:
\begin{align}
    U(\widehat{P}^{(t)}, \tau_t) = \widehat{P}^{(t)} + \alpha_{\text{adapt}} \nabla \sum_{s,a,s' \in \tau_t} \log(\widehat{P}^{(t)}(s'|s,a)).
\label{eq:fewshot}
\end{align}

After we have the estimation of the current dynamics, we can construct our policy $\widehat{\pi}^{(t)}$, which is defined by running model predictive control methods such as CEM \cite{de2005tutorial}, MPPI \cite{williams2017information} or tree search \cite{schrittwieser2020mastering} on the current adapted dynamics model $U(\widehat{P}^{(t)}, \tau_t)$. Thus we denote the policy as 
$\widehat{\pi}^{(t)} = \text{MPC}(U(\widehat{P}^{(t)}, \tau_t)).$

With the constructed policy $\widehat{\pi}^{(t)}$, we then collect $K$ trajectories inside the current $\Mcal^{(t)}$: with probability $\frac{1}{2}$ we will follow the current policy $\widehat{\pi}^{(t)}$, and with probability $\frac{1}{2}$ we will follow the exploration policy $\pi^{e}$.\footnote{Note that the probability $\frac{1}{2}$ here is for the ease of theoretical analysis. In fact, one can change the probability of rolling out with $\pi^e$ to arbitrary non-zero probability and will only change a constant coefficient on the final result.} For simplicity, we denote the state-action distribution induced by the exploration policy as $\nu_t := d_{P^{(t)}}^{\pi^e}$, which will be used for analysis in the next section. Finally after we store all the trajectories we obtained in the current iteration into dataset $D_t$, along with samples from previous iterations, we can perform no-regret online meta learning algorithm (e.g., Follow The Meta Leader) to obtain the new meta model $\widehat{P}^{(t+1)}$:
\begin{align}
\widehat{P}^{(t+1)} = \argmax_{P \in \Pcal} \sum_{n=1}^t \sum_{s,a,s' \in D_n} \log U(P, \tau_n)(s'|s,a).
\label{eq:ftml}
\end{align}
Here we treat the model class and the meta model class as the same class since they are defined to have the same parametrization.

Thus the proposed method integrates online learning, meta learning and  MBRL and is better than the previous algorithms in the following sense: comparing with online SystemID methods, our method can return a set of personalized policies when incoming environments have different dynamics rather than policy for one fixed environment. Comparing with meta MBRL methods, our method can learn to adapt in an \textit{online} manner under extensive real-human user-study, where previous methods only work in simulation. Also meta RL needs to be trained on a set of tasks before testing so it is not applicable for our online setting. Another advantage of our method is that it has theoretical guarantee comparing with the meta RL methods, which we will show in the next section.

%% file: analysis_short.tex

\section{Analysis}
\label{sec:analysis}

In this section we provide the regret analysis for Alg.~\ref{alg:alg_box}. At a high level, our regret analysis is based on the previous analysis of the online SystemID framework \cite{ross2012agnostic} that applies DAgger \cite{ross2011reduction} in the MBRL setting to train the sequence of models, where DAgger can be regarded as a no-regret online learning procedure such as Follow-the-Leader (FTL) algorithm \cite{shalev2011online}. In our setting, on the other hand, our model training framework could be regarded as Follow-the-Meta-Leader (FTML) \cite{finn2019online} that is also no-regret with our designed training objective. The overall intuition is that with the sequence of more reliable adapted models, the regret can be bounded and diminish over time. We first introduce our assumptions:
\begin{assum}[Optimal Control Oracle]
We assume that the model predictive control methods return us an $\epsilon-\text{optimal}$ policy. That is, for any transitional dynamics kernel $P$ and initial state distribution $\mu$, we have
\[\EE_{s \sim \mu} \left[ V^{P,\text{MPC}(P)}(s) \right] - \min_{\pi \in \Pi} \EE_{s \sim \mu} \left[ V^{P,\pi}(s) \right] \leq \epsilon_{\text{oc}}. \]
\end{assum}


\begin{assum}[Agnostic]
We assume agnostic setting. That is, our model class may not contain the ground truth meta model that can perfectly adapt to any dynamics. Formally, given arbitrary state-action distribution $d$, one trajectory $\tau$ induced by $P^{(t)}$, $\exists \widehat{P}^\ast \in \mathcal{P}$, such that $\forall t \in [T]$,
\[\EE_{s,a \sim d} \left[D_{KL}(U(\widehat{P}^\ast, \tau)(\cdot|s,a), P^{(t)}(\cdot|s,a)) \right] \leq \epsilon_{model}.\]
\label{assump:model}
\end{assum}
Note that $\epsilon_{\text{oc}}$ and $\epsilon_{\text{model}}$ are independent from our algorithm (i.e., we can not eliminate this term no matter how carefully we design our algorithm).


Our analysis begins with bounding the performance difference between the sequence of policies generated by our algorithm and any policy $\pi'$, which can be the policy induced by adapting from the best meta-model. As the first step, we can relate the performance difference to the model error with: 1. simulation lemma (Lemma~\ref{lem:simulation}) \cite{kearns2002near,sun2019model}. 2. We relate the trajectories induced by $\pi'$ and $\pi^e$ by the coefficient term introduced in the Sec.~\ref{sec:pre}. Note that since we should only care about comparing with ``good" $\pi'$, the coefficient is small under such circumstances.

The next step is thus to bound the model error over the online process. Let's first observe the following loss function:
$
    \ell^{(t)}(P) = \EE_{s,a \sim \rho_t} \left[D_{KL}\left( P(\cdot|s,a), P^{(t)}(\cdot|s,a) \right) \right], 
$
where $\rho_t$ defined as the state-action distribution of our data collection scheme. This correspond to the model error. Then if we use FTML to update the meta model:
$
    \widehat{P}^{(t+1)} = \argmin_{P \in \Pcal} \sum_{i=1}^t \ell^{(t)}(U(P, \tau_t)),
$
it is easy to see this update is equivalent to Eq.~\ref{eq:ftml} as in our algorithm. Then we can invoke the main result from \cite{finn2019online} (Lemma~\ref{lemma:ftml}), and we are ready to present our main result:
\begin{theorem}
Let $\{\widehat{P}^{(t)}\}_{t=1}^T$ be the learned meta models. Let $\{U(\widehat{P}^{(t)}, \tau_t)\}_{t=1}^T$ be the adapted model after the one-shot adaptations. Let $\boldsymbol{\widehat{\pi}} = \{\widehat{\pi}^{(t)}\}_{t=1}^T$, where $\widehat{\pi}^{(t)} := MPC(U(\widehat{P}^{(t)}, \tau_t))$. Let $\rho_t := \frac{1}{2}d_{P^{(t)}}^{\widehat{\pi}^{(t)}} + \frac{1}{2}\nu_t$ be the state-action distribution induced by our algorithm under $\Mcal^{(t)}$. Then for policy sequence $\boldsymbol{\pi'} = \{\pi'^{(t)}\}_{t=1}^T$, we have
 \footnote{Note that $\Tilde{O}$ hides the logarithmic terms.}
\begin{align*}
     J_{\mu}(\boldsymbol{\widehat{\pi}}) - J_{\mu}(\boldsymbol{\pi'})  \leq \epsilon_{oc} + \max_t(c_{\nu_t}^{\pi'^{(t)}})H^2 \sqrt{\epsilon_{\text{model}}} 
+\Tilde{O}\left( \frac{\max(c_{\nu_t}^{\pi'^{(t)}}) H^2}{\sqrt{T}} \right).
\end{align*}
\label{thm:main}
\end{theorem}
We defer the full proofs in Appendix~\ref{app:proofs}. Thm.~\ref{thm:main} shows that given any $\boldsymbol{\pi'}$, which could be the policies induced by the best meta model in the model class, the performance gap incurred by our algorithm diminishes with rate $\Tilde{O}(\frac{1}{\sqrt{T}})$, along with the inevitable terms $\epsilon_{\text{oc}}$ and $\epsilon_{\text{model}}$, which our algorithm has no control over. Since our main result is based on any arbitrary policy sequence, we here first formally define the best adapted sequence of policies from hindsight, and we could easily conclude with the no-regret conclusion in the following corollary:
\begin{corollary}
Let $\boldsymbol{\widehat{\pi}} = \{\widehat{\pi}^{(t)}\}_{t=1}^T$ be the policies returned by our algorithm. Define $\pi^{\ast}_{P} = \argmin_{\pi \in \Pi} \\ \EE_{s \sim \mu}V^{\pi,P}(s)$. Let $\boldsymbol{\pi^{\ast}}_P = \{\pi^{\ast}_{U(P,\tau_t)}\}_{t=1}^T$. For simplicity let $\pi^{\ast(t)} = \pi^{\ast}_{U(P,\tau_t)}$. Let $R(T) = J(\boldsymbol{\widehat{\pi}}) - \min_{P \in \Pcal}J(\boldsymbol{\pi^{\ast}}_P)$ be the regret of our algorithm. We have
\begin{align*}
    \lim_{T \to \infty} R(T) = \epsilon_{oc} +\max_t(c_{\nu_t}^{\pi^{\ast(t)}})H^2 \sqrt{\epsilon_{\text{model}}}.
\end{align*}
\end{corollary}

%% file: exp.tex
\section{Experiments}
\subsection{System Setup}
\label{exp:set_up}

Our vehicle navigation system is developed inside the Carla simulator \cite{dosovitskiy2017carla}. Each navigation run follows the below procedure: a vehicle spawns at a specified location inside the simulator, and the navigation system provides one audio instruction to the user every one second (45 frames). The user will control the vehicle with either a keyboard or a driving kit, and react to the audio instructions, until the vehicle is near the specified destination point. The goal is to compare the reliability of different navigation algorithms based on optimal path tracking and safety. 

We develop the navigation system in 5 Carla native town maps, 
which include comprehensive traffic situations. 
We use Carla's built-in A* algorithm to search for the optimal path: we first manually specify a starting point and destination point, which define a route, and the A* planner returns a list of waypoints, defined by its x,y,z coordinates, that compose the optimal path.  

To introduce various and comprehensive dynamics, we invite human participants with different levels of skills and propose three ways to induce diverse dynamics: a) different built-in vehicles, b) different visual conditions, c) different cameras (points of view), as shown in Fig.~\ref{fig:dynamics_changes}. 

Finally, we introduce our design for the state and action spaces for our RL setting. We discretize the system by setting every 45 frames (around 1s) as one time step. We design each route to have a length of 2 minutes on average, thus each trajectory contains 120 state-action pairs on average. Each state is a 14-dimensional vector that contains low-level information about the vehicle, user inputs, and the immediate goal. For the action space, we adopt a discrete action space given the limited audio generation and modification functionality of Carla's interface. Nevertheless, we observe that our current design suffices to deliver timely instructions combined with our proposed algorithm. In fact, our algorithm could also expand the action space, which we will demonstrate in detail in Sec~\ref{exp:new_action}. Concretely, one action corresponds to one audio instruction, and the action space contains the 7 instructions in total, including the no instruction option.  In practice, we only play the corresponding audio when the current action is different from the previous action, or the current action has been repeated for 15 timesteps.

We defer a complete list of system details in Appendix~\ref{app:system}: the town maps can be found in Sec.~\ref{app:system:map}, state representation in Sec.~\ref{app:system:state}, actions in Sec.~\ref{app:system:action}, list of different vision conditions, vehicles, driving kits in Sec.~\ref{app:system:iter}.
\subsection{Practical Implementation}
In this section we introduce the practical implementation details. There are two major specifications for a successful transfer from theory to practice: the model parametrization and the MPC method.

For the model parametrization, we represent the (meta) model class $\Pcal$ as a class of deep feedforward neural networks. Each element in the class can be denoted as $P_{\theta}$, where here $\theta$ is the parameters of $P$. Here we adopt a deterministic model given there is no stochasticity in the underlying Carla system. Furthermore, since some instructions have lasting affect on system dynamics (e.g., prepare to turn), we feed a history of state-action pairs to the network for robust predictions. Concretely, the network input is a history of state-action pairs $h_t = \{s_{t-k}, a_{t-k}\}_{k=0}^K$ and the output of the network is the predicted next state $\widehat{s}_{t+1}$. In practice using $K=4$ suffices to include enough information for model prediction. In addition, we discover that predicting state differences in the egocentric vehicle coordinates boosts the accuracy of the model. 

Because of the deterministic system, maximizing log-likelihood is equivalent to minimizing the $l_2$ norms of the state differences. In this work, we follow the empirical success of L-step loss \cite{nagabandi2018deep, luo2018algorithmic}. Specifically, with a batch $\{h_t,s_{t+1}\}_{t=1}^B$, where $B$ is the batch size, for $P_{\theta}$, we update the model by

\begin{align}
\theta := \theta - \alpha_{\text{meta}} \sum_{i=1}^{B} \nabla_{\theta}  \left(\sum_{l=1}^L \left\|  (\hat{s}_{i+l} - \hat{s}_{i+l-1}) -( s_{i+l} - s_{i+l-1}) \right\|_2\right),
\label{eq:mini_batch}
\end{align}

\noindent where $\hat{h}_i = h_i$ and $\hat{s}_{i+l+1} = U(P_{\theta} ({\hat{h}_{i+l}, a_{i+l}}))$, and $\hat{h}_{i+l+1}$ is constructed by discarded the first state action pair in $\hat{h}_{i+l+1}$ and concatenate $\hat{s}_{i+l+1}$ and $a_{i+l+1}$ to its end. In our experiments, we use $L = 5$. Similarly, the parameter of $U(P_{\theta})$ is obtained by using Eq.~\ref{eq:mini_batch} with the adaptation sample (line 3 in Alg.~\ref{alg:alg_box}) where next states are predicted by $P_{\theta}$.

For MPC, since the action space is discrete, we adopt a tree search scheme, although the current design can be easily switched to MPPI or CEM. We define the cost of an action sequence $\{a_h\}_{h=1}^H$ under $P_{\theta}$ by:
$
    C(\{a_h\}_{h=1}^H, P_{\theta}) = \sum_{h=1}^H \|\hat{s}_{h+1} - g_{h+1}\|_2^2,
$
where $\{\hat{s}_{h+1}\}_{h=1}^H$ are obtained by autoregressively rolling out $P_{\theta}$ with $\{a_h\}_{h=1}^H$ and $\{g_{h+1}\}_{h=1}^H$ are the coordinates of goal waypoints. Note on the right hand side we abuse notation by letting $\hat{s}$ also be coordinates (but not the full states). Intuitively, we want the model-induced policy to be a residual policy of the explore policy, since the explore policy is accurate (but not personalized enough) for most situations. Thus we reward an action sequence that is the same as the explore policy action sequence (but at potentially shifted time intervals). We also add penalties for turning actions where there are no junctions. 

\subsection{Experimental Results}
\label{exp:results}
We evaluate our practical algorithm inside the system described in Sec.~\ref{exp:set_up}. For each of the 5 towns, we design 10 routes that contain all traffic conditions. We define using one navigation algorithm to run the 10 routes in one town as one iteration. For each of the towns, we first use the explore policy to collect 3 iterations of samples as the offline dataset. For the online process, the first 7 iterations involve different vehicles and visual conditions. In 3 out of 5 towns, we invite additional 6 participants with diverse driving skills to complete the next 6 iterations. Then for all the maps, we include 2 additional iterations that are controlled by different control equipment. Each human participant contributes at least 3-hour driving inside the system, and in total more than 60-hour driving data is used for our final evaluation.

\begin{figure*}[t]
    \centering
    \includegraphics[width=0.312\textwidth]{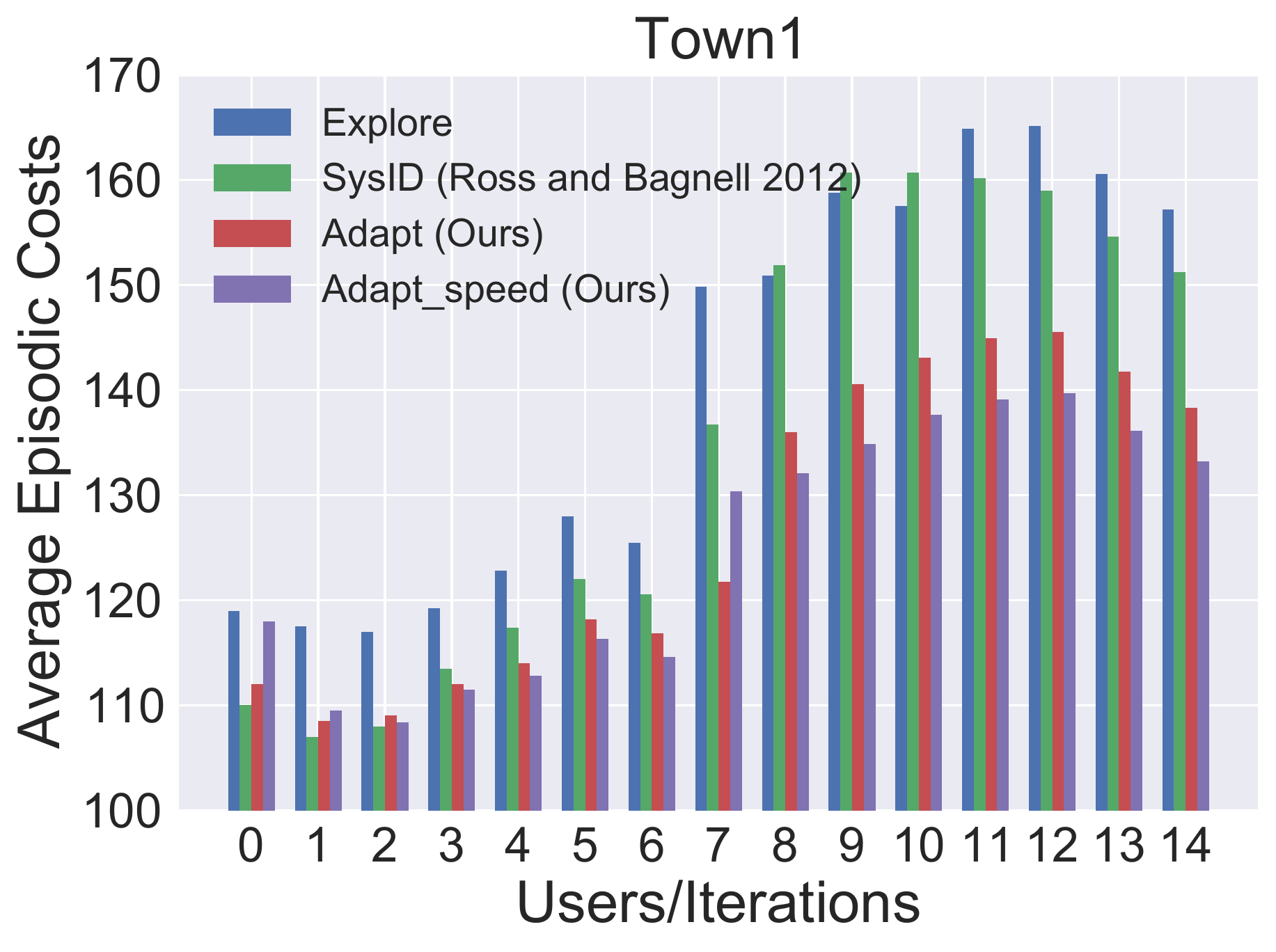}
    \includegraphics[width=0.312\textwidth]{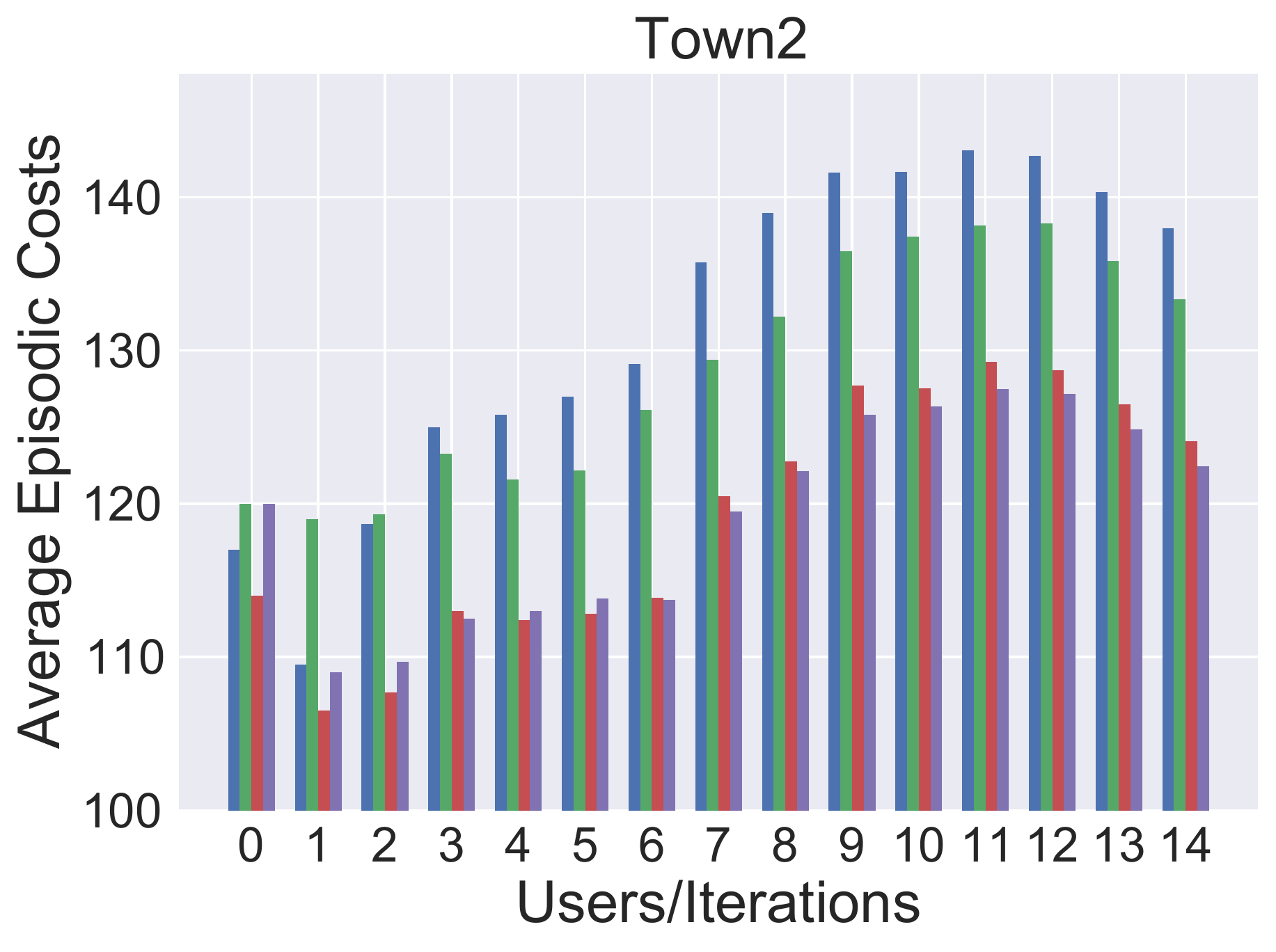}
    \includegraphics[width=0.312\textwidth]{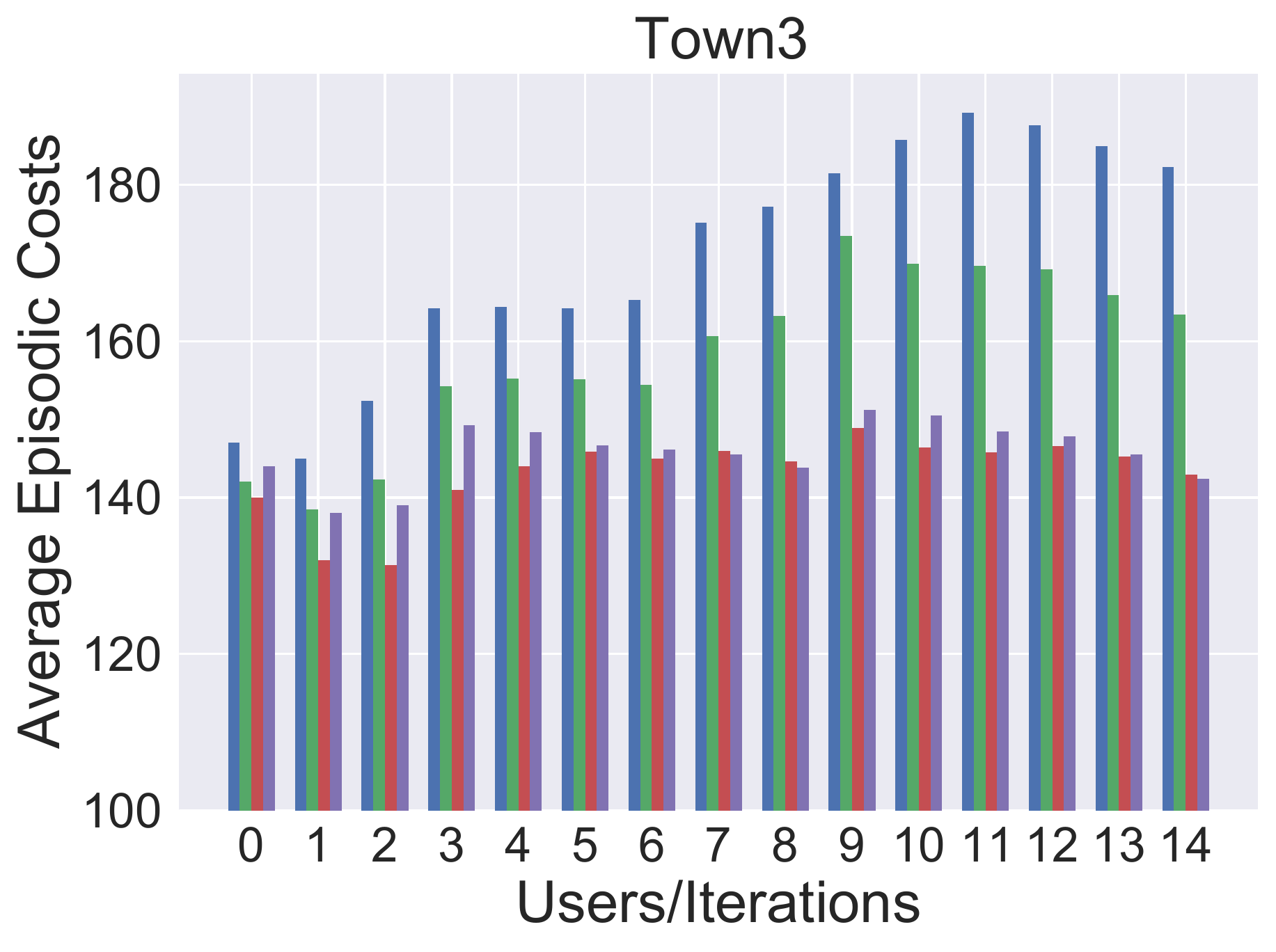}
    \caption{Performance bar plots of the four baselines in three town maps. The metric we use is the average episodic accumulative costs upon the current iteration. Here the episodic accumulative costs are calculated by averaging the point-wise tracking error over the ten routes that we defined in Sec.~\ref{exp:set_up}}
    \label{fig:results}
\end{figure*}
We include three baselines in this study: the first is the \textit{explore policy}\footnote{The name ``explore policy'' inherits from \cite{ross2012agnostic}, but here the explore policy is actually a very strong baseline, as described in the main text. 
}, a non-learning-based policy based on A* path planning algorithm. The underlying backbone of the explore policy is the Carla's autopilot agent, which provides expert demonstrations for the state-of-the-art autonomous driving algorithms \cite{chen2020learning, chen2021learning, prakash2021multi}. To increase the accuracy and robustness of the policy, we further include hard-coded heuristics for the corner cases. The second baseline is a policy induced by the non-adaptive model trained online \cite{ross2012agnostic}. We refer to this policy as \textit{SysID} in the rest of the sections. The model training of SysID follows Eq.~\ref{eq:mini_batch} with only one exception that the next states are predicted directly by the models instead of the adapted models (i.e., no meta adaptation). The third baseline is our algorithm and we include a variation of our algorithm as well, which will be introduced in detail in Sec.~\ref{exp:new_action}.

The metric we use in this section is the point-wise tracking error. Concretely, for a state $s_t$, the cost we are evaluating is $c(s_t) = \|x_{s_t} - x_{g_{t-1}}, y_{s_t} - y_{g_{t-1}} \|_2^2$,  where $x,y$ denotes the coordinates elements inside a state and $g_{t-1}$ denotes the goal selected from the waypoint list from the last timestep. 
We follow Sec.~\ref{sec:pre} and plot the empirical performance $\hat{J}(\boldsymbol{\pi}) = \frac{1}{T}\sum_{t=1}^T H_t^{\pi^{(t)}}$, where $H_t$ is the sum of costs in one episode. We show the performance of the 4 baselines in 3 town maps that involve human participants with different levels of driving skills in Fig.~\ref{fig:results} (we defer the results of the remaining two maps in Appendix~\ref{app:exp}). We observe that except for the first online iteration, our method (denoted as adapt) beats the other two baselines (explore and SysID) consistently in tracking the optimal routes. There are two possible reasons that the navigation system could accumulate huge costs: an ill-timed instruction could cause the drive to miss a turn at a junction thus follow a suboptimal route. On the other hand, if the navigation system is not adaptive to the driver, mildly ill-timed instructions may still keep the driver on the optimal route, but deviate from the waypoints (e.g., driving on the sidewalk). This will also raise safety issues as we will further evaluate in Sec.~\ref{exp:safety}. Note that our results indicate the significance of personalization: non-adaptive methods, even with online updates (SysID), behave suboptimally and we observe the performance gaps between non-adaptive methods and our method increase as the more users are involved in the system. We also observe that dynamics change induced by different users is much more challenging than simulation-based dynamics change as real human participants introduce significantly greater dynamics shifts than simulation-based changes, as indicated in Fig.\ref{fig:results} starting from user/iteration 7.

\subsection{Generating New Actions}
\label{exp:new_action}
In this section we analyze a new variant of our adaptive policy: adaptive policy with speed change. The new policy is designed to generate new actions by changing the playback speed of the original instructions: slower, normal and faster. Thus the new policy increases the action space dimensions by three times. We include the performance of this new variant along with other baselines in Fig.~\ref{fig:results}, denoted as ``Adaptive\_speed". We see that this new policy design performs similarly to the original adaptive policy at the first few iterations due to the lack of samples with speed change, but it gradually outperforms as the online procedure goes on.

\subsection{Safety}
\label{exp:safety}

In addition to better tracking, our empirical evaluation also shows that our method has better safety guarantees. In this section, we present the number of collisions for each method in Table~\ref{table:safety}. In each cell we show the total number of collisions averaged over the 10 routes in the corresponding town maps. We then average over the 15 online iterations and present the mean and standard deviation. We observe that comparing with the explore policy baseline, our methods (adaptive and adaptive speed) achieve over 50\% and 60\% decrease in terms of the number of collisions, respectively. The results indicate that by better tracking the optimal path and delivering more timely instructions with personalized policies, the adaptive policies indeed achieve a stronger safety guarantee. 

\subsection{Empirical Regret}
\label{exp:emp_regret}

In this section, we analyze the empirical regret of the adaptive policy, i.e., how our adaptive policy compares with the policy induced by the best empirical meta model from hindsight. We train the hindsight policy using the data collected from all 15 iterations and train with Eq.~\ref{eq:ftml}. We compare the two policies in Fig.~\ref{fig:empirical_regret}, following the same metric as in Fig.~\ref{fig:results}. The results show that the performance gap exists at the beginning, but it closes up as the online process proceeds. Here we want to emphasize that in practice, this empirical regret (performance gap) may not diminish strictly at a rate of $T^{-1/2}$ since we could only expect such rate in expectation, we do observe that such gap decreases over time and in some iteration, the performance of the adaptive policy is actually the same as the best policy from hindsight, indicating the strong adaptability of our proposed method.\\

\noindent
\begin{minipage}[t]{.43\textwidth}
\raggedright
    \centering
    \includegraphics[width=\textwidth]{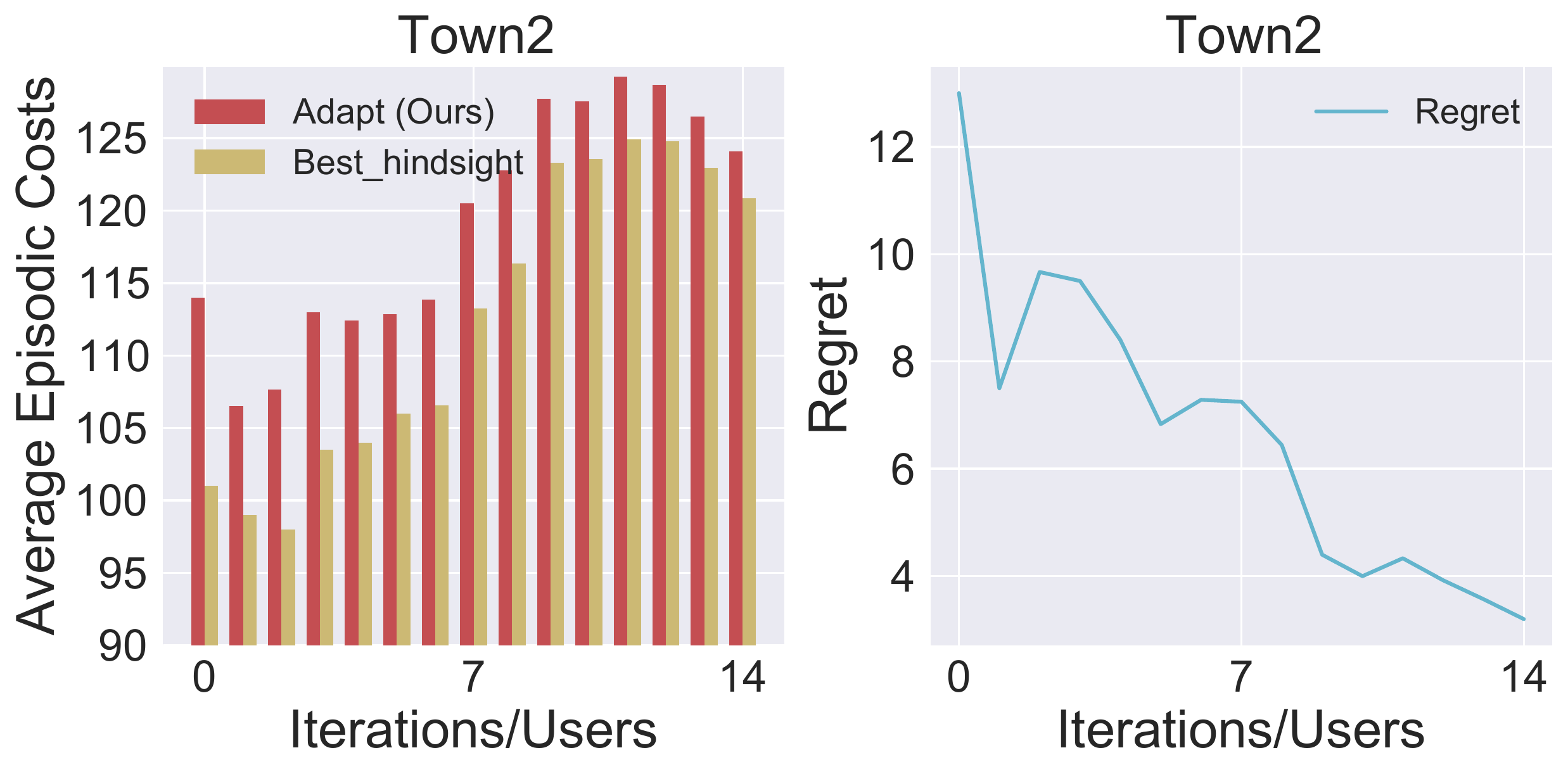}
    \captionof{figure}{Evaluation of empirical regret/performance gap of the adaptive policy in town 2. The evaluation metric and presentation follow Fig.~\ref{fig:results}.}
    \label{fig:empirical_regret}
\end{minipage}%
\hfill
\begin{minipage}[t]{.53\textwidth}
\raggedleft
   \vspace{-3.05cm}
\renewcommand{\arraystretch}{1.5}
 \centering  \resizebox{\columnwidth}{!}{
    \begin{tabular}{|l|c|c|c|} 
        \hline
        &  Town1  & Town2 & Town3   \\
        \hline
        Explore policy      &  3.08 (1.73) &  1.00 (1.18) &  1.38 (1.21)  \\
        \hline
        SysID            &  2.23 (1.19) &  0.69 (1.14) &  1.23 (1.19)  \\
        \hline
        Adapt   (Ours)         &  1.85 (1.29) &  0.31 (0.61) &  0.69 (0.72)  \\
        \hline
        Adapt Speed (Ours)    &  1.54 (1.45) &  0.23 (0.42) &  0.54 (0.50)\\
        \hline
        \end{tabular}}
      \captionof{table}{Average counts of total collisions of each baseline in the 3 town maps. The average is over 15 online iterations and the standard deviation is indicated in the parenthesis.}
      \label{table:safety}
\end{minipage}


%% file: appendix.tex


\section{Proofs}
\label{app:proofs}
In this section we prove Thm.~\ref{thm:main}. We first introduce the simulation lemma:
\begin{lemma}[Simulation Lemma] Consider two MDPs $\Mcal_1 = \{C, P\}$ and $\Mcal_2 = \{C', P'\}$ where $C$ and $P$ represent cost and transition. Let $J(\pi;C,P) = \EE_{s\sim \mu}\left[V^{\pi,P,c}(s)\right]$, where the superscript $C$ denotes that the value function $V$ is the expected sum of cost $C$, and for simplicity let $V^{\pi} = V^{\pi,P,C}$. Then for any policy $\pi:\Scal\times\Acal\mapsto \Delta(\Acal)$, we have:
\begin{align*}
J(\pi; C, P) - J(\pi; C', P') = \sum_{h=0}^{H-1} \mathbb{E}_{s,a\sim d_h^{\pi}} \left[ C(s,a) - C'(s,a) +  \mathbb{E}_{s'\sim P(\cdot | s,a)}V^{\pi}_h(s') - \EE_{s'\sim P'(\cdot|s,a)} V^{\pi}_h(s') \right].
\end{align*}\label{lem:simulation}
\end{lemma}
We omit the proof since the simulation lemma is widely used in MBRL analysis. We refer the reader to lemma 10 of \cite{sun2019model} for an example.

Then let's start off with an easier problem, where we bound the performance difference between two single policies: 

\begin{lemma}
\label{lemma:perf_model}
Let $\widehat{\pi}, \pi' \in \Pi$, where $\widehat{P} \in \Pcal$ and $\widehat{\pi} = MPC(\widehat{P})$. Let $\nu$ be an exploration distribution, we have:
\[J_{\mu}(\widehat{\pi}) - J_{\mu}(\pi') \leq \epsilon_{oc} + H^2(c_{\nu}^{\widehat{\pi}} + c_{\nu}^{\pi'}) \mathbb{E}_{s,a\sim \nu} \left [\|\widehat{P}(\cdot | s,a) - P(\cdot | s,a)\|_1 \right] \]
\end{lemma}
\begin{proof}
For simplicity, denote $\widehat{V}^{\pi} = V^{\pi,\widehat{P}}$, we have
\begin{align*}
    J_{\mu}(\widehat{\pi}) - J_{\mu}(\pi') &= \EE_{s \sim \mu}[V^{\widehat{\pi}}(s) - V^{\pi'}(s)] \\
    &= \EE_{s \sim \mu}[V^{\widehat{\pi}}(s) - \widehat{V}^{\widehat{\pi}}(s)] + 
    \EE_{s \sim \mu}[\widehat{V}^{\widehat{\pi}}(s) - \widehat{V}^{\pi'}(s)] +
    \EE_{s \sim \mu}[\widehat{V}^{\pi'}(s) - V^{\pi'}(s)] \\
    &= \EE_{s \sim \mu}[\widehat{V}^{\widehat{\pi}}(s) - \widehat{V}^{\pi'}(s)] +
    \EE_{s \sim \mu}[V^{\pi'}(s) - \widehat{V}^{\widehat{\pi}}(s)] + 
    \EE_{s \sim \mu}[\widehat{V}^{\pi'}(s) - V^{\pi'}(s)] \\
    &\leq \epsilon_{oc} + \sum_{h=0}^{H-1} \mathbb{E}_{s,a\sim d_h^{\widehat{\pi}}} \left[ \mathbb{E}_{s'\sim \widehat{P}(\cdot | s,a)}\widehat{V}^{\widehat{\pi}}_h(s') - \EE_{s'\sim P(\cdot|s,a)} \widehat{V}^{\widehat{\pi}}_h(s') \right] \\ &\qquad
     + \sum_{h=0}^{H-1} \mathbb{E}_{s,a\sim d_h^{\widehat{\pi}}} \left[ \mathbb{E}_{s'\sim \widehat{P}(\cdot | s,a)}\widehat{V}^{\pi'}_h(s') - \EE_{s'\sim P(\cdot|s,a)} \widehat{V}^{\pi'}_h(s') \right] \\
     & \leq \epsilon_{oc} + (c_{\nu}^{\widehat{\pi}} + c_{\nu}^{\pi'})  \sum_{h=0}^{H-1} \mathbb{E}_{s,a\sim \nu} \left[ \mathbb{E}_{s'\sim \widehat{P}(\cdot | s,a)}\widehat{V}^{\pi'}_h(s') - \EE_{s'\sim P(\cdot|s,a)} \widehat{V}^{\pi'}_h(s') \right] \\
     & \leq  \epsilon_{oc} + (c_{\nu}^{\widehat{\pi}} + c_{\nu}^{\pi'})  \sum_{h=0}^{H-1} \mathbb{E}_{s,a\sim \nu} \left [\|\widehat{P}(\cdot | s,a) - P(\cdot | s,a)\|_1 \|\widehat{V}^{\pi'}_h\|_{\infty}\right] \\
     & \leq  \epsilon_{oc} + H^2(c_{\nu}^{\widehat{\pi}} + c_{\nu}^{\pi'})  \mathbb{E}_{s,a\sim \nu} \left [\|\widehat{P}(\cdot | s,a) - P(\cdot | s,a)\|_1 \right] \\
\end{align*}
The first inequality is by simulation lemma, the second is by the definition of the coverage coefficient, and the third is by Holder's inequality, and the last is by the assumption that the costs are bounded by 1.
\end{proof}

With the previous lemma, we are ready to prove Lemma~\ref{lemma:perf_model}:
\begin{lemma} \label{lemma:perf_model} Let $\{\widehat{P}^{(t)}\}_{t=1}^T$ be the learned meta models. Let $\{U(\widehat{P}^{(t)}, \tau_t)\}_{t=1}^T$ be the adapted model after the one-shot adaptations. Let $\boldsymbol{\widehat{\pi}} = \{\widehat{\pi}^{(t)}\}_{t=1}^T$, where $\widehat{\pi}^{(t)} := MPC(U_t(\widehat{P}^{(t)}))$. Let $\rho_t := \frac{1}{2}d_{P^{(t)}}^{\widehat{\pi}^{(t)}} + \frac{1}{2}\nu_t$ be the state-action distribution induced by our algorithm under $\Mcal^{(t)}$. Then for policy sequence $\boldsymbol{\pi'} = \{\pi'^{(t)}\}_{t=1}^T$, we have:
\begin{align*}
    J_{\mu}(\boldsymbol{\widehat{\pi}}) - J_{\mu}(\boldsymbol{\pi'})  &\leq {\epsilon}_{oc} +  2\max_t(c_{\nu_i}^{\pi'^{(t)}})H^2\frac{1}{T} \sum_{t=1}^{T}  \mathbb{E}_{s,a\sim \rho_t} \|U(\widehat{P}^{(t)}, \tau_t)(\cdot | s,a) 
    - P^{(t)}(\cdot | s,a)\|_1 
\end{align*}
\end{lemma}\begin{proof}
\begin{align*}
     J_{\mu}(\boldsymbol{\widehat{\pi}}) - J_{\mu}(\boldsymbol{\pi'}) &= \frac{1}{T} \sum_{t=1}^{T} \left( J_{\mu}(\widehat{\pi}^{(t)}) - J_{\mu}(\pi'^{(t)}) \right) \\
    &\leq \epsilon_{oc} + \frac{H}{T}  \sum_{t=1}^{T} \sum_{h=0}^{H-1}\bigg( \mathbb{E}_{s,a\sim d^{\widehat{\pi}^{(t)}}_{h,P^{(t)}}} \left [\|U(\widehat{P}^{(t)}, \tau_t)(\cdot | s,a) - P^{(t)}(\cdot | s,a)\|_1 \right] \\ &\qquad\qquad\qquad+ \mathbb{E}_{s,a\sim d^{\pi'}_{h,P^{(t)}}} \left [\|U(\widehat{P}^{(t)}, \tau_t)(\cdot | s,a) - P^{(t)}(\cdot | s,a)\|_1 \right]\bigg)\\
    & \leq \epsilon_{oc} + \frac{H}{T}  \sum_{t=1}^{T} \sum_{h=0}^{H-1}\bigg( \mathbb{E}_{s,a\sim d^{\widehat{\pi}^{(t)}}_{h,P^{(t)}}} \left [\|U(\widehat{P}^{(t)}, \tau_t)(\cdot | s,a) - P^{(t)}(\cdot | s,a)\|_1 \right] \\ &\qquad\qquad\qquad+ c_{\nu_t}^{\pi'^{(t)}} \mathbb{E}_{s,a\sim \nu_t} \left [\|U(\widehat{P}^{(t)}, \tau_t)(\cdot | s,a) - P^{(t)}(\cdot | s,a)\|_1 \right]\bigg)\\
    & \leq \epsilon_{oc} + \frac{H}{T}  \sum_{t=1}^{T} \sum_{h=0}^{H-1}\bigg( c_{\nu_t}^{\pi'^{(t)}}  \mathbb{E}_{s,a\sim d^{\widehat{\pi}^{(t)}}_{h,P^{(t)}}} \left [\|U(\widehat{P}^{(t)}, \tau_t)(\cdot | s,a) - P^{(t)}(\cdot | s,a)\|_1 \right] \\ &\qquad\qquad\qquad+ c_{\nu_t}^{\pi'^{(t)}} \mathbb{E}_{s,a\sim \nu_t} \left [\|U(\widehat{P}^{(t)}, \tau_t)(\cdot | s,a) - P^{(t)}(\cdot | s,a)\|_1 \right]\bigg)\\
    & = \epsilon_{oc} + \frac{2H^2}{T}  \sum_{t=1}^{T} c_{\nu_t}^{\pi'^{(t)}}  \sum_{h=0}^{H-1}\bigg(   \frac{1}{2H}\mathbb{E}_{s,a\sim d^{\widehat{\pi}^{(t)}}_{h,P^{(t)}}} \left [\|U(\widehat{P}^{(t)}, \tau_t)(\cdot | s,a) - P^{(t)}(\cdot | s,a)\|_1 \right] \\ &\qquad\qquad\qquad+  \frac{1}{2H} \mathbb{E}_{s,a\sim \nu_t} \left [\|U(\widehat{P}^{(t)}, \tau_t)(\cdot | s,a) - P^{(t)}(\cdot | s,a)\|_1 \right]\bigg)\\
    &=  \epsilon_{oc} + \frac{2H^2}{T}  \sum_{t=1}^{T} c_{\nu_t}^{\pi'^{(t)}} \bigg(\mathbb{E}_{s,a\sim \rho_t} \left [\|U(\widehat{P}^{(t)}, \tau_t)(\cdot | s,a) - P^{(t)}(\cdot | s,a)\|_1 \right] \bigg)\\
    & \leq \epsilon_{oc} + \frac{2\max_t(c_{\nu_t}^{\pi'^{(t)}})H^2}{T}  \sum_{t=1}^{T}  \left(\mathbb{E}_{s,a\sim \rho_t} \left [\|U(\widehat{P}^{(t)}, \tau_t)(\cdot | s,a) - P^{(t)}(\cdot | s,a)\|_1 \right] \right)\\
\end{align*}
\end{proof}

\begin{lemma}[Regret bound of FTML]
Suppose that $\ell^{(t)}$ are strongly convex. With $\alpha_{\text{adapt}}$ and $\alpha_{\text{meta}}$ properly chosen, FTML has the regret
\begin{align*}
    R_{FTML}(T) = \sum_{t=1}^T \ell^{(t)}(U(\widehat{P}^{(t)}, \tau_t)) - \min_{P \in \Pcal} \sum_{t=1}^T \ell^{(t)}(U(P, \tau_t)) \leq O(\log(T)).
\end{align*}
\label{lemma:ftml}
\end{lemma}
For proof we refer the reader to \cite{finn2019online}.

The final step left to prove the main theorem remains to bound the model error. With the above lemma, we show another lemma that leads to the main result:

\begin{lemma}
Let $\widehat{\pi}^{(t)} := MPC\left(U(\widehat{P}^{(t)},\tau_t)\right)$. Let $\rho_t := \frac{1}{2}d_{P^{(t)}}^{\widehat{\pi}^{(t)}} + \frac{1}{2}\nu_t$ be the state-action distribution induced by our algorithm under $\Mcal^{(t)}$. We have
\[ \frac{1}{T} \sum_{t=1}^T \mathbb{E}_{s,a\sim \rho_t} \left [\|U(\widehat{P}^{(t)},\tau_t)(\cdot | s,a) - P^{(t)}(\cdot | s,a)\|_1 \right]  \leq \Tilde{O}\left(\frac{1}{\sqrt{T}}\right).\] 
\label{lemma:combine_model_error}
\end{lemma}

\begin{proof}
Let the loss function $\ell^{(t)}$ be 
\[\ell^{(t)}(P) = \EE_{s,a \sim \rho_t} \left[ \EE_{s' \sim P^{(t)}(\cdot|s,a)}\left[ -\log(P(s'|s,a)) \right] \right].\]
Since $\{\widehat{P}^{(t)}\}_{t=1}^T$ is obtained by running FTML on the sequence of strongly convex function $\{\ell^{(t)}\}_{t=1}^T$, we have by Lemma~\ref{lemma:ftml} that 
\begin{align*}
    \sum_{t=1}^T \ell^{(t)}\left(U(\widehat{P}^{(t)},\tau_t)\right) & \leq \min_{P \in \mathcal{P}}\sum_{t=1}^T \ell^{(t)}\left(U(\widehat{P}^{(t)},\tau_t)\right) + O(\log(T)) \\
    \sum_{t=1}^T \ell^{(t)}\left(U(\widehat{P}^{(t)},\tau_t)\right) + \EE_{s,a \sim \rho_t} \EE_{s' \sim P(\cdot|s,a)} \log(P^{(t)}(s'|s,a)) & \leq \min_{P \in \mathcal{P}}\sum_{t=1}^T \ell^{(t)}\left(U(\widehat{P}^{(t)},\tau_t)\right) \\&\qquad +  \EE_{s,a \sim \rho_t} \EE_{s' \sim P(\cdot|s,a)} \log(P^{(t)}(s'|s,a)) + O(\log(T)) \\
    \sum_{t=1}^T \EE_{s,a \sim \rho_t} D_{KL} \left(U(\widehat{P}^{(t)},\tau_t)(\cdot | s,a),P^{(t)}(\cdot | s,a)\right) & \leq \min_{P \in \mathcal{P}} \sum_{t=1}^T \EE_{s,a \sim \rho_t} D_{KL} \left(U(P,\tau_t)(\cdot | s,a),P^{(t)}(\cdot | s,a)\right) \\& \;\; + O(\log(T)) \\
    \sum_{t=1}^T \EE_{s,a \sim \rho_t}  D_{KL} \left(U(\widehat{P}^{(t)},\tau_t)(\cdot | s,a),P^{(t)}(\cdot | s,a)\right) & \leq T\epsilon_{\text{model}}+  O(\log(T)), \\
\end{align*}
where the last line is obtained by Assumption~\ref{assump:model}.

Finally by  Pinsker’s inequality and Jensen's inequality, 
\begin{align*}
    \frac{1}{T} \sum_{t=1}^T \mathbb{E}_{s,a\sim \rho_t}  \EE_{s' \sim P(\cdot|s,a)}\left [\|U(\widehat{P}^{(t)},\tau_t)(s' | s,a) - P^{(t)}(s' | s,a)\|_1 \right]  & \leq  \frac{1}{T} \sum_{t=1}^T\sqrt{2 \EE_{s,a \sim \rho_t} D_{KL} \left(U(\widehat{P}^{(t)}, \tau_t),P^{(t)}\right)} \\
    & \leq \sqrt{ \frac{1}{T} \sum_{t=1}^T 2 \EE_{s,a \sim \rho_t} D_{KL} \left(U(\widehat{P}^{(t)},\tau_t),P^{(t)}\right)} \\
    &\leq \sqrt{\epsilon_{\text{model}}} +  \Tilde{O}\left(\frac{1}{\sqrt{T}}\right).
\end{align*}   
\end{proof}

Then apply Lemma~\ref{lemma:combine_model_error} on Lemma~\ref{lemma:perf_model}, we obtain our main theorem:
\begin{theorem}[Main Theorem]
Let $\{\widehat{P}^{(t)}\}_{t=1}^T$ be the learned meta models. Let $\{U(\widehat{P}^{(t)}, \tau_t)\}_{t=1}^T$ be the adapted model after the one-shot adaptations. Let $\boldsymbol{\widehat{\pi}} = \{\widehat{\pi}^{(t)}\}_{t=1}^T$, where $\widehat{\pi}^{(t)} := MPC(U_t(\widehat{P}^{(t)}))$. Let $\rho_t := \frac{1}{2}d_{P^{(t)}}^{\widehat{\pi}^{(t)}} + \frac{1}{2}\nu_t$ be the state-action distribution induced by our algorithm under $\Mcal^{(t)}$. Then for any meta policy $\pi'$, we have:
\begin{align*}
     J_{\mu}(\boldsymbol{\widehat{\pi}}) - J_{\mu}(\boldsymbol{\pi'})  \leq \epsilon_{oc} + \max_t(c_{\nu_t}^{\pi'^{(t)}})H^2 \sqrt{\epsilon_{\text{model}}} +\Tilde{O}\left( \frac{\max_t(c_{\nu_t}^{\pi'^{(t)}}) H^2}{\sqrt{T}} \right)
\end{align*}
\end{theorem}
\begin{proof}
This is a direct combination of Lemma~\ref{lemma:perf_model} and Lemma~\ref{lemma:combine_model_error}.
\end{proof}
Finally we can use our main theorem to reach the no-regret conclusion:
\begin{corollary}
Let $\boldsymbol{\widehat{\pi}} = \{\widehat{\pi}^{(t)}\}_{t=1}^T$ be the policies returned by our algorithm. Let $R(T) = J(\boldsymbol{\widehat{\pi}}) - \min_{\pi^\ast \in \Pi}J(\boldsymbol{\pi^\ast})$ be the regret of our algorithm. We have
\begin{align*}
    \lim_{T \to \infty} R(T) = \epsilon_{oc} +\max_t(c_{\nu_t}^{\pi^{\ast(t)}})H^2 \sqrt{\epsilon_{\text{model}}}.
\end{align*}
\end{corollary}
\begin{proof}
Let $\pi^\ast = min_{\pi^\ast \in \Pi}J(\boldsymbol{\pi^\ast})$. Replace $\boldsymbol{\pi'}$ by $\boldsymbol{\pi^\ast}$ in Thm.~\ref{thm:main} we have
\begin{align*}
     R(T) = \epsilon_{oc} +\max_t(c_{\nu_t}^{\pi^{\ast(t)}})H^2 \sqrt{\epsilon_{\text{model}}} + \Tilde{O}\left( \frac{\max_t(c_{\nu_t}^{\pi^{\ast(t)}}) H^2}{\sqrt{T}} \right).
\end{align*}
Taking $T \to \infty$ completes the proof.
\end{proof}

\newpage
\section{Algorithm}
In this section, we provide the pseudocode of our algorithm.
\label{app:alg}

\begin{algorithm}[h]
	\begin{algorithmic}[1]	
		\REQUIRE  exploration policy $\pi^e$, offline data $D_{\text{off}}$
		\STATE Get warm-started model $\widehat{P}^{(1)}$ with Eq.\ref{eq:offline}
        \FOR{$t = 1, \dots, T$}
            \STATE Collect one trajectory $\tau_t = \{s_h, a_h, s_{h+1} | a_h \sim \pi^e, s_{h+1} \sim P^{(t)}(s_h,a_h) \}_{h=1}^H$
            \STATE Get $U(\widehat{P}^{(t)}, \tau_t)$ with Eq.~\ref{eq:fewshot}.
            \STATE $\widehat{\pi}^{(t)} \xleftarrow{}$ MPC$\left(U(\widehat{P}^{(t)}, \tau_t)\right)$
            \STATE $D_{t} = \{\}$
            \FOR {$k=1, \dots, K$}
                \STATE With prob. $\frac{1}{2}$ rollout using current policy $\widehat{\pi}^{(t)}$: 
                 \STATE $\tau_k = \{s_h, a_h, s_{h+1} | a_h \sim \widehat{\pi}^{(t)}, s_{h+1} \sim P^{(t)}(\cdot|s_h,a_h) \}_{h=1}^H$
                \STATE Otherwise rollout using exploration  policy $\pi^e$: 
                \STATE $\tau_k = \{s_h, a_h, s_{h+1} | a_h \sim \pi^e, s_{h+1} \sim P^{(t)}(\cdot|s_h,a_h) \}_{h=1}^H$
                \STATE $D_{t} \xleftarrow{} D_{t} \cup \tau_k$.
            \ENDFOR
            \STATE Get $\widehat{P}^{(t+1)}$ by Eq.~\ref{eq:ftml}.
        \ENDFOR
        \end{algorithmic}
        \caption{Online Meta Model Based RL}
        \label{alg:alg_box}
\end{algorithm}

\newpage
\section{System Details}
\label{app:system}
\subsection{Town maps}
\label{app:system:map}
In this section, we show the town maps that are used in the development of our system in Fig.~\ref{fig:town_map}.
\begin{figure}[H]
    \centering
    \includegraphics[width=0.19\textwidth]{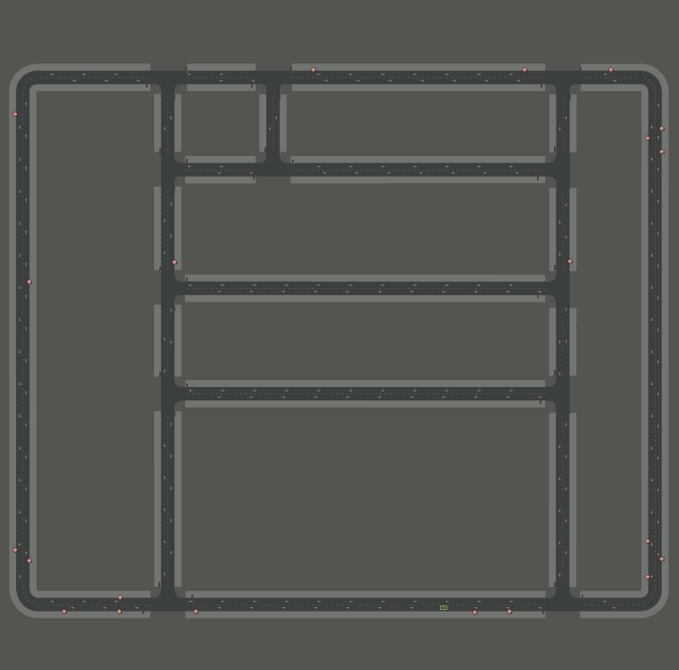}
    \includegraphics[width=0.19\textwidth]{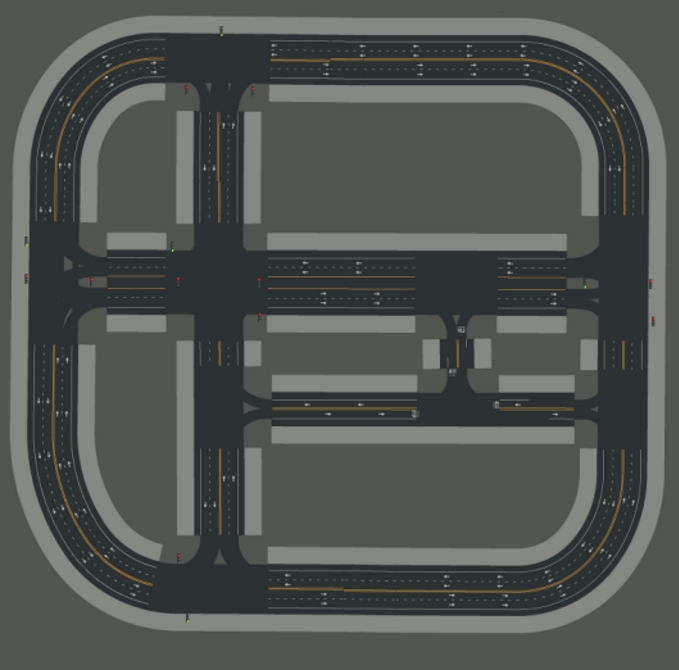}
    \includegraphics[width=0.19\textwidth]{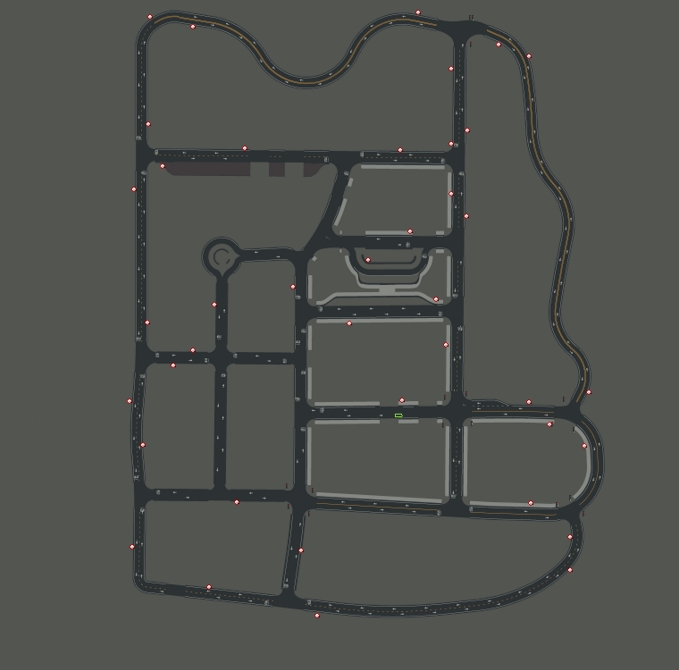}
    \includegraphics[width=0.19\textwidth]{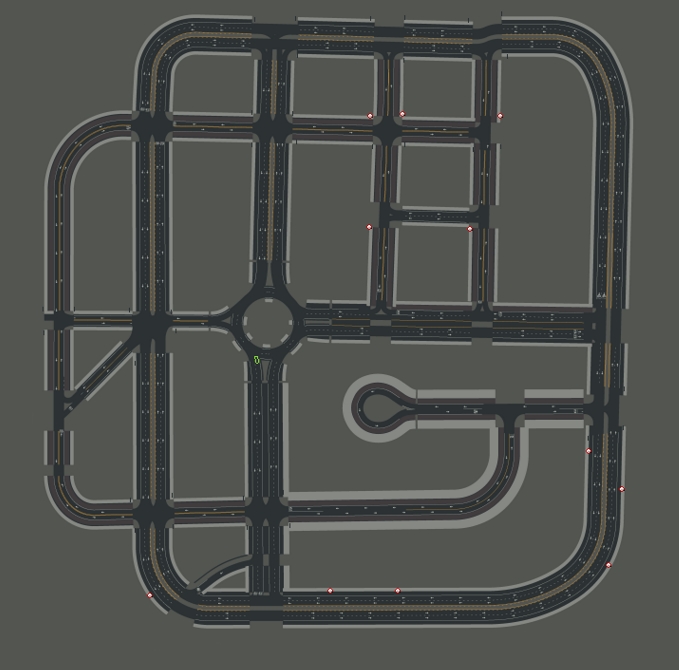}
    \includegraphics[width=0.19\textwidth]{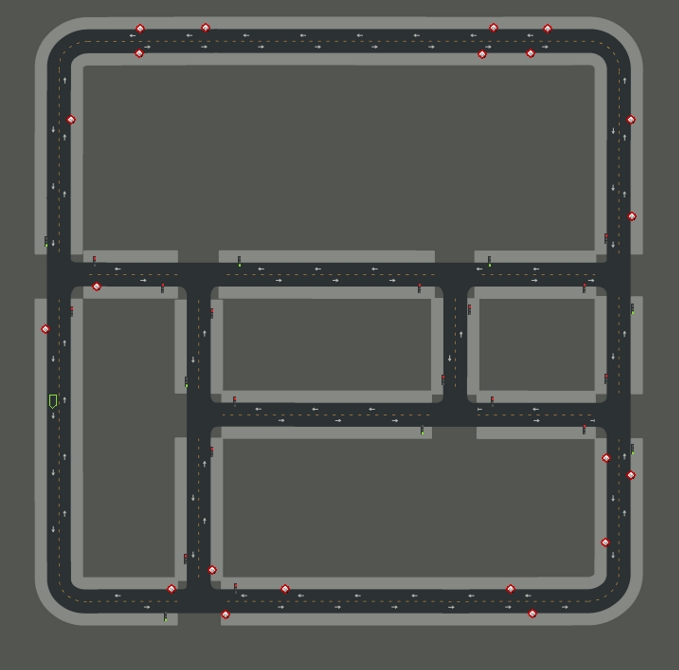}
    \caption{From left to right: the maps of Town1, Town2, Town3, Town4 and Town5 in our experiments.}
    \label{fig:town_map}
\end{figure}

\subsection{State Representation}
\label{app:system:state}
In this section, we describe the details of our state representation in our navigation system. At each time step, the state is a concatenation of the following attributes:
\begin{itemize}
    \item the current (x,y,z) coordinates of the vehicle. 
    \item The current velocity in the x,y,z direction. 
    \item The current heading (in radius) of the vehicle.
    \item An indicator $i = e^{-d}$, where $d$ is the distance to the next junction along the route. 
    \item A boolean indicating whether the vehicle is at a roundabout. 
    \item The (x,y) coordinates of the next waypoint in the route list return by the A* planner, which is selected according to the current speed of the vehicle. 
    \item The steering input from the user.
    \item A boolean indicating whether the vehicle is in reverse mode. 
    \item The braking input from the user. 
\end{itemize}

\subsection{List of instructions}
\label{app:system:action}
\begin{itemize}
    \item No instruction. 
    \item Turn left. 
    \item Turn right.
    \item Prepare to turn left. 
    \item Prepare to turn right.
    \item Enter the roundabout. 
    \item Keep left (in the roundabout).
\end{itemize}

\subsection{Iteration Details}
\label{app:system:iter}
In this section, we introduce the environment (user, vehicle, visual condition, etc.) involved in each iteration in Table~\ref{table:iterations}. 
\begin{table}[h] 
\renewcommand{\arraystretch}{1.26}
\centering{
\begin{tabular}{|p{1.7cm}|p{0.9cm}|c|c|p{1.7cm}|p{2.4cm}|} 
\hline
Iteration index & User Index   &  Vehicle  & Camera & Visual Condition & Control Equipment   \\
\hline
Offline 0 & 1 & Chevrolet Camaro & First personal view & Normal & Keyboard \\
\hline
Offline 1 & 1 & Chevrolet Camaro & Third personal view & Normal & Keyboard \\
\hline
Offline 2 & 1 & Tesla Cybertruck & First personal view & Normal & Keyboard \\
\hline
Online 0 & 1 & Tesla Cybertruck & Third personal view & Normal & Keyboard \\
\hline
Online 1 & 1 & Chevrolet Camaro & First personal view & RGB & Keyboard \\
\hline
Online 2 & 1 & Chevrolet Camaro & Third personal view & Foggy & Keyboard \\
\hline
Online 3 & 1 & Chevrolet Camaro & First personal view & Foggy & Keyboard \\
\hline
Online 4 & 1 & Tesla Cybertruck & Third personal view & Foggy & Keyboard \\
\hline
Online 5 & 1 & Tesla Cybertruck & First personal view & Foggy & Keyboard \\
\hline
Online 6 & 1 & Tesla Cybertruck & Top view & Foggy & Keyboard \\
\hline
Online 7 & 2 & Chevrolet Camaro & Third personal view & Normal & Keyboard \\
\hline
Online 8 & 3 & Chevrolet Camaro & Third personal view & Normal & Keyboard \\
\hline
Online 9 & 4 & Chevrolet Camaro & Third personal view & Normal & Keyboard \\
\hline
Online 10 & 5 & Chevrolet Camaro & Third personal view & Normal & Keyboard \\
\hline
Online 11 & 6 & Chevrolet Camaro & Third personal view & Normal & Keyboard \\
\hline
Online 12 & 7 & Chevrolet Camaro & Third personal view & Normal & Keyboard \\
\hline
Online 13 & 1 & Chevrolet Camaro & First personal view & Normal & Logitech G29 \\
\hline
Online 14 & 1 & Chevrolet Camaro & Third personal view & Normal & Logitech G29 \\
\hline
\end{tabular}}
\caption{}
\label{table:iterations}
\end{table}

\newpage
\section{Additional Experiments}
\label{app:exp}
In this section, we provide results from the two remaining town maps. The evaluations are the same as in Sec.~\ref{exp:results} and Sec.~\ref{exp:safety}. We observe that the results from the additional experiments follow the trend in Sec.~\ref{exp:results} and Sec.~\ref{exp:safety}.
\subsection{Tracking Error}
\begin{figure}[h]
    \centering
    \includegraphics[width=0.45\textwidth]{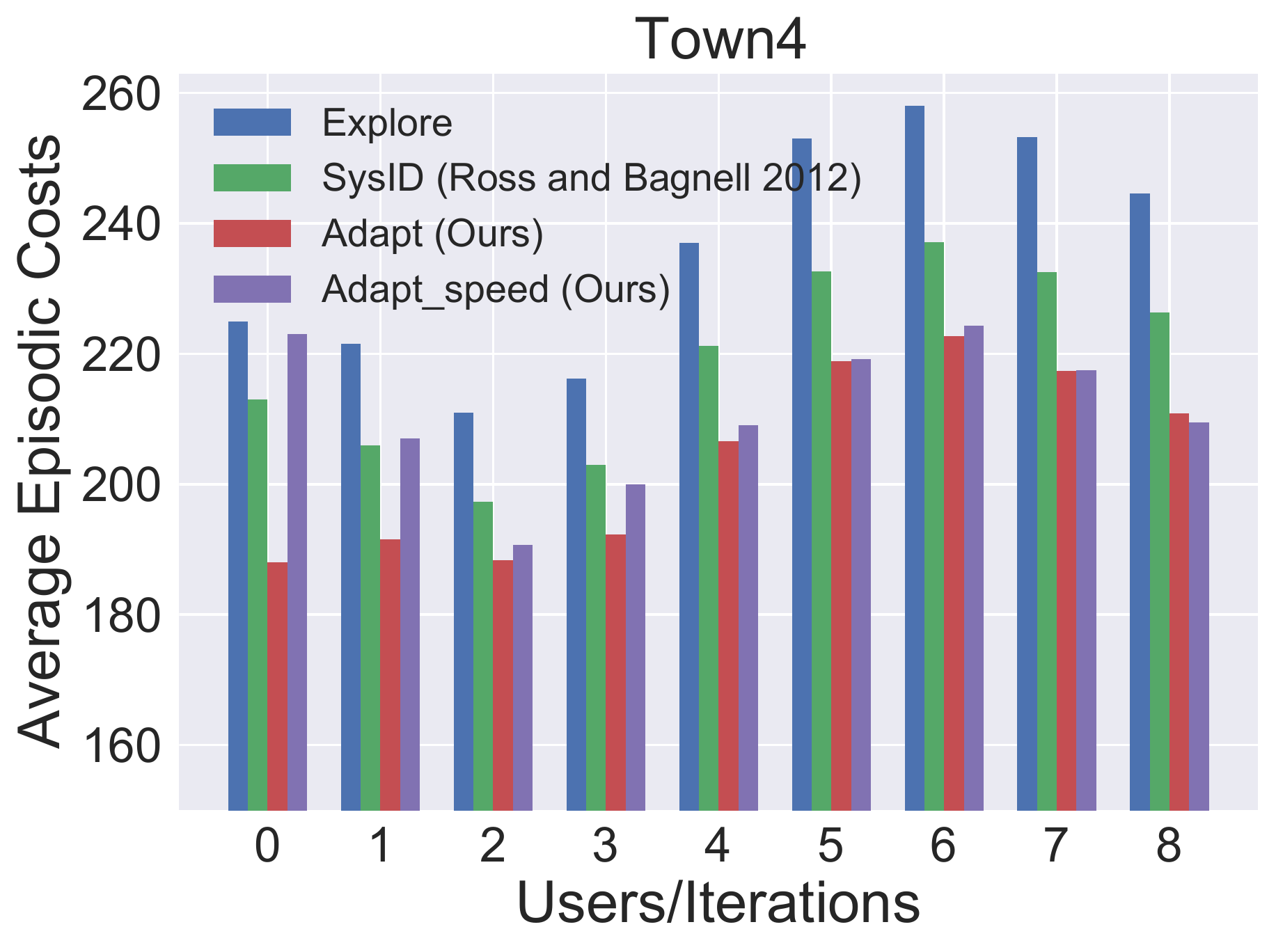}
    \includegraphics[width=0.45\textwidth]{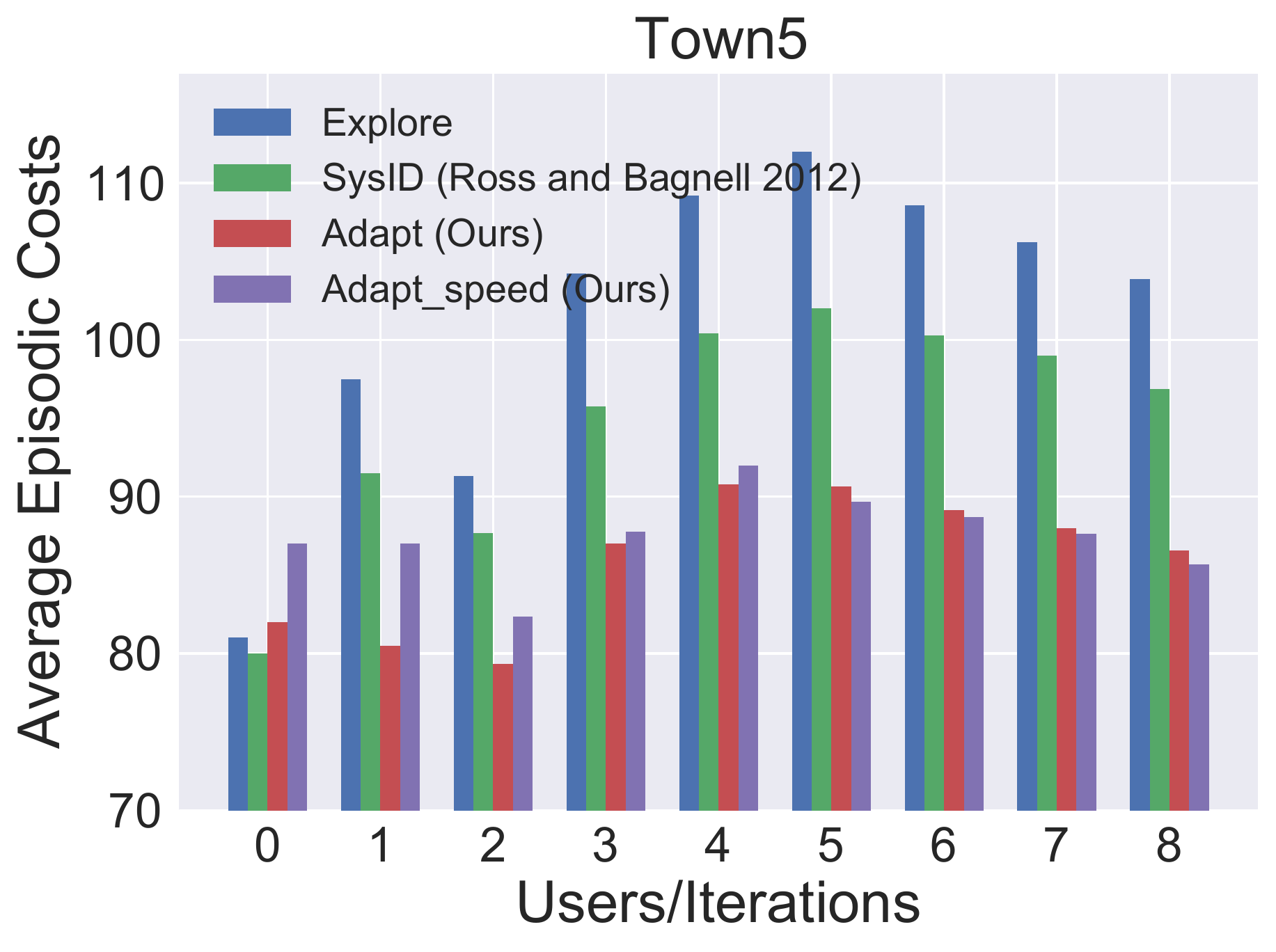}
    \caption{Performance bar plots of the four baselines in the remaining two town maps. The metric we use is the average episodic accumulative costs upon the current iteration. Here the episodic accumulative costs are calculated by averaging the point-wise tracking error over the ten routes that we defined in Sec.~\ref{exp:set_up}}
    \label{app:exp:fig:results}
\end{figure}

\subsection{Safety}
\begin{table}[h] 
\centering{
\renewcommand{\arraystretch}{1.2}
\begin{tabular}{|l|c|c|c|} 
\hline
                    &  Town4  & Town5    \\
\hline
Explore policy      &  4.00 (0.67) &  2.89 (0.99)   \\
\hline
SysID (Ross and Bagnell 2012)              &  3.00 (0.82) &  2.11 (0.74)  \\
\hline
Adaptive   (Ours)         &  2.22 (0.63) &  1.22 (0.92)   \\
\hline
Adaptive Speed  (Ours)    &  1.89 (0.74) &  1.11 (0.74) \\
\hline
\end{tabular}}
\caption{Average counts of total collision of each baseline in the remaining two town maps. The average is over the 9 online iterations and the standard deviation is indicated in the parenthesis.}
\label{app:table:safety}
\end{table}

\newpage
\section{Implementation Details}
\label{app:imp_detail}
\subsection{Hyperparameters}
In this section we provide our hyperparameters in Table~\ref{app:table:hyperparam}.
\begin{table}[h] 
\centering
\begin{tabular}{lp{4.0cm}l} 
\toprule
                                         & Value Considered     & Final Value \\ 
\hline
Model Learning Rate                      & \{1e-3, 5e-3, 1e-4\} & 1e-3        \\
Dynamics Model Hidden Layer Size         & \{[500,500],[500,500,500],      & [256,256,256]   \\
& [256,256],[256,256,256]\}   &  \\
Number of Model Updates                  & \{150,100,50\}       & 50          \\
Multistep Loss L                         & \{3,4,5\}            & 5           \\
Batch Size                               & \{128,256\}          & 128         \\
History Length                           & \{3,4,5\}            & 5           \\
Adaptation Learning Rate                 & \{0.1,0.05,0.01, 0.001\}    & 0.01 \\
MPC rollout horizon                      & \{3,4,5,6\}          & 5           \\
\toprule
\end{tabular}
\caption{Hyperparameters.}
\label{app:table:hyperparam}
\end{table}

\subsection{Details on user study}
In this section, we provide more details on the user study. We since the users will drive all 4 baselines across the same routes, a principled pipeline is required to ensure fairness of the comparison. Each user will first be trained on the control of the system, the training time varies among the users depending on their expertise in driving and keyboard controls. The user will first collect the trajectory for adaptation using the explore policy for each map. Then we randomize the order of the four baselines, and the user will drive one baseline across the 3 maps (30 routes in total). Then the user will take a 30 minutes break and proceed to the next baseline. Depending on the user's availability, we either invite the user back the next day or take a very long break after the second baseline. Then we finish the third and fourth baselines with another 30 minutes break in between. The users are agnostic about the baselines they are currently using. We conduct the experiment this way for two reasons: first, the experiment itself is very intensive, and thus we want to avoid performance drop of any baselines due to user's fatigue. Another reason is that we utilize break time to avoid users having any memories with the previous routes to ensure fairness of the comparison. Finally, for the experiment in Sec~\ref{exp:emp_regret}, we invite all the users back to drive with the best model from hindsight, except this time we don't enforce the order of users since we don't update the model anymore.